%% file: lin-bandits.tex
\documentclass[twoside]{article} \usepackage{style} 

\newif\ifanon
\anonfalse

\newif\ifsup
\suptrue

\usepackage{fullpage}
\usepackage{latexsym}
\usepackage{amsmath}
\usepackage{amssymb}
\usepackage{mathtools}
\usepackage{enumerate}
\usepackage{accents}
\usepackage{tikz}
\usepackage{pgfplots}
\usepackage[ruled]{algorithm}
\usepackage{algpseudocode}
\usepackage{dsfont}
\usepackage[bf,font=scriptsize]{caption}
\usepackage[disable,textsize=tiny]{todonotes}
\newcommand{\todoc}[2][]{\todo[size=\scriptsize,color=blue!20!white,#1]{Csaba: #2}}
\newcommand{\todot}[2][]{\todo[size=\scriptsize,color=red!20!white,#1]{Tor: #2}}
\newcommand{\smallsc}[1]{\text{\scalebox{0.9}{\scshape #1}}}
\usepackage{hyperref}
\hypersetup{
    bookmarks=true,         
    unicode=false,          
    pdftoolbar=true,        
    pdfmenubar=true,        
    pdffitwindow=false,     
    pdfstartview={FitH},    
    pdftitle={The End of Optimism},    
    pdfauthor={Lattimore and Szepesvari},     
    pdfsubject={Bandits},   
    pdfcreator={Creator},   
    pdfproducer={Producer}, 
    pdfkeywords={keyword1} {key2} {key3}, 
    pdfnewwindow=true,      
    colorlinks=true,       
    linkcolor=red,          
    citecolor=blue,        
    filecolor=magenta,      
    urlcolor=cyan           
}
\usepackage{amsthm}
\usepackage{times}
\usepackage{natbib}
\usepackage{nicefrac}
\usepackage{wrapfig}
\usepackage{enumitem}
\usepackage[capitalize]{cleveref}

\let\epsilon\varepsilon

\theoremstyle{plain}
\newtheorem{theorem}{Theorem}

\newtheorem{lemma}[theorem]{Lemma}
\newtheorem{corollary}[theorem]{Corollary}
\theoremstyle{definition}
\newtheorem{definition}[theorem]{Definition}

\newtheorem{remark}[theorem]{Remark}
\newtheorem{example}[theorem]{Example}
\theoremstyle{remark}

\newcommand\numberthis{\addtocounter{equation}{1}\tag{\theequation}}

\newcommand{\E}{\mathbb E}
\newcommand{\Ep}{\E'}
\newcommand{\EE}[1]{\E\left[#1\right]}
\newcommand{\EEp}[1]{\E'\left[#1\right]}
\newcommand{\Var}{\operatorname{Var}}

\newcommand{\set}[1]{\left\{#1\right\}}
\newcommand{\ind}[1]{\mathds{1}\!\set{#1}}
\newcommand{\argmax}{\operatornamewithlimits{arg\,max}}

\newcommand{\ceil}[1]{\left \lceil {#1} \right\rceil}
\newcommand{\eqn}[1]{\begin{align}#1\end{align}}
\newcommand{\eq}[1]{\begin{align*}#1\end{align*}}

\newcommand{\norm}[1]{\left\Vert #1 \right\Vert}
\newcommand{\shortnorm}[1]{\Vert #1 \Vert}
\newcommand{\R}{\mathbb R}
\newcommand{\N}{\mathbb N}

\newcommand{\inner}[1]{\left<#1\right>}
\newcommand{\shortinner}[1]{\langle #1 \rangle}
\newcommand{\ip}[1]{\shortinner{#1}}
\renewcommand{\P}[1]{\mathbb P \left(#1\right)}
\newcommand{\Pp}[1]{\mathbb P' \left(#1 \right)}
\newcommand{\KL}{\operatorname{KL}}
\newcommand{\laspan}{\operatorname{span}}
\newcommand{\trace}{\operatorname{tr}}

\def\ddefloop#1{\ifx\ddefloop#1\else\ddef{#1}\expandafter\ddefloop\fi}

\def\ddef#1{\expandafter\def\csname b#1\endcsname{\ensuremath{\mathbf{#1}}}}
\ddefloop ABCDEFGHIJKLMNOPQRSTUVWXYZ\ddefloop

\def\ddef#1{\expandafter\def\csname bb#1\endcsname{\ensuremath{\mathbb{#1}}}}
\ddefloop ABCDEFGHIJKLMNOPQRSTUVWXYZ\ddefloop

\def\ddef#1{\expandafter\def\csname c#1\endcsname{\ensuremath{\mathcal{#1}}}}
\ddefloop ABCDEFGHIJKLMNOPQRSTUVWXYZ\ddefloop

\def\ddef#1{\expandafter\def\csname v#1\endcsname{\ensuremath{\boldsymbol{#1}}}}
\ddefloop ABCDEFGHIJKLMNOPQRSTUVWXYZabcdefghijklmnopqrstuvwxyz\ddefloop

\def\ddef#1{\expandafter\def\csname v#1\endcsname{\ensuremath{\boldsymbol{\csname #1\endcsname}}}}
\ddefloop {alpha}{beta}{gamma}{delta}{epsilon}{varepsilon}{zeta}{eta}{theta}{vartheta}{iota}{kappa}{lambda}{mu}{nu}{xi}{pi}{varpi}{rho}{varrho}{sigma}{varsigma}{tau}{upsilon}{phi}{varphi}{chi}{psi}{omega}{Gamma}{Delta}{Theta}{Lambda}{Xi}{Pi}{Sigma}{varSigma}{Upsilon}{Phi}{Psi}{Omega}\ddefloop

\begin{document}

\twocolumn[

\aistatstitle{The End of Optimism? \\ An Asymptotic Analysis of Finite-Armed Linear Bandits}

\ifanon
\aistatsauthor{Anonymous Author 1 \And Anonymous Author 2}

\aistatsaddress{Unknown Institution 1 \And Unknown Institution 2}
\else

\aistatsauthor{Tor Lattimore \And Csaba Szepesv\'ari}

\aistatsaddress{Indiana University, Bloomington \And University of Alberta, Edmonton}

\fi


]

\begin{abstract}
Stochastic linear bandits are a natural and simple generalisation of finite-armed bandits with numerous practical applications. 
Current approaches focus on generalising 
existing techniques for finite-armed bandits, notably the optimism principle and Thompson sampling. While prior work
has mostly been in the worst-case setting, we analyse the asymptotic instance-dependent regret and show matching upper and lower bounds on what is achievable. Surprisingly, our
results show that no algorithm based on optimism or Thompson sampling will ever achieve the optimal rate, and indeed, can be arbitrarily far
from optimal, even in very simple cases. This is a disturbing result because these techniques are standard tools that are widely used for sequential optimisation.
For example, for generalised linear bandits and reinforcement learning.
\end{abstract}

\section{INTRODUCTION}\label{sec:intro}

The linear bandit is the simplest generalisation of the finite-armed bandit.
Let $\mathcal A \subset \R^d$ be a finite set that spans $\R^d$ with $|\mathcal A| = k$ and $\norm{x}_2 \leq 1$ for all $x \in \cA$. 
A learner interacts with the bandit over $n$ rounds. In each round $t$
the learner chooses an action (arm) $A_t \in \mathcal A$ and observes a payoff $Y_t = \inner{A_t, \theta} + \eta_t$ where $\eta_t \sim \mathcal N(0,1)$
is Gaussian noise and $\theta \in \R^d$ is an unknown parameter. 
The optimal action is $x^* = \argmax_{x \in \mathcal A} \inner{x, \theta}$, which is not known since it depends on $\theta$.
The assumption that $\cA$ spans $\R^d$ is non-restrictive, since if $\laspan(\cA)$ has rank $r < d$, then one can simply use a different basis for which
all but $r$ coordinates are always zero and then drop them from the analysis. The Gaussian assumption can be relaxed to $1$-subgaussian for our upper bound, but
is needed for the lower bound.
Our performance measure is the expected pseudo-regret (from now on just the regret), which is given by 
\eq{
R^\pi_\theta(n) = \E\left[\sum_{t=1}^n \inner{x^* - A_t, \theta}\right]\,,
}
where the expectation is taken with respect to the actions of the strategy and the noise.
There are a number of algorithms designed for minimising the regret, all of which use one of
two algorithmic designs. The first is the principle of optimism in the face of uncertainty, which was originally applied to finite-armed bandits by
\cite{Agr95,KR95,ACF02} and many others, and more recently to linear bandits \citep{Aue02,DHK08,AST11,APS12}.
The second algorithm design is Thompson sampling, which is an old algorithm \citep{Tho33} that has experienced a resurgence in popularity because of its
impressive practical performance and theoretical guarantees for finite-armed bandits \citep{KKM12,KKM13}.
Thompson sampling has also recently been applied to linear bandits with good empirical performance \citep{CL11} and near-minimax theoretical guarantees \citep{AG13}.

While both approaches lead to practical algorithms (especially Thompson sampling), we will show they are fundamentally
flawed in that algorithms based on these ideas cannot be close to asymptotically optimal.
Along the way we characterise the optimal achievable asymptotic regret and design a strategy achieving it.
This is an important message because optimism and Thompson sampling are widely used beyond the finite-armed case.
Examples include generalised linear bandits \citep{FCGS10}, spectral bandits \citep{VMKK14}, and even learning in Markov decision processes \citep{AJO10,GM15}.

The disadvantages of these approaches is obscured in the worst-case regime, where both are quite close to optimal. One might question whether or not the asymptotic analysis
is relevant in practice. The gold standard would be instance-dependent finite-time guarantees like what is available for finite-armed bandits, but
historically the asymptotic analysis has served as a useful guide towards understanding the trade-offs in finite-time. Besides hiding the structure of specific problems, pushing
for optimality in the worst-case regime can also lead to sub-optimal instance-dependent guarantees. For example, the MOSS algorithm for finite-armed bandits is minimax optimal, but
far from finite-time optimal \citep{AB09}. For these reasons we believe that understanding the asymptotics of a problem is a useful first step towards optimal finite-time instance-dependent
guarantees that are most desirable.

It is worth mentioning that partial monitoring (a more complicated online learning setting) is a well known example of the failure of optimism \citep{BFPRS14}.
Although related, the partial monitoring framework is more general than the bandit setting because the learner may not 
observe the reward even for the action they take, which means that additional exploration is usually necessary in order to gain information. 
Basic results in partial monitoring are concerned with characterizing whether an instance is easier or harder than bandit instances.
More recently, the question of asymptotic instance optimality was studied in finite stochastic partial monitoring  \citep{KHN15}, and the special setting of learning with side information  \citep{Wu15}. While the algorithms derived in these works served as inspiration, the analysis and the algorithms do not generalise in a simple direct fashion to the linear setting, which requires a careful study of how information is transferred between actions in a linear setting.

\section{NOTATION}\label{sec:notation}

For positive semidefinite $G$ (written as $G\succeq 0$) and vector $x$ we write $\norm{x}_G^2 = x^\top G x$.
The Euclidean norm of a vector $x\in \R^d$ is $\norm{x}$ and
the spectral norm of a matrix $A$ is $\norm{A}$. 
The pseudo-inverse of a matrix $A$ is denoted by $A^\dagger$.
The mean of arm $x \in \cA$ is $\mu_x = \inner{x, \theta}$ and the optimal mean is $\mu^* = \max_{x \in \mathcal A} \mu_x$. Let $x^*\in \cA$ be 
any \emph{optimal action} such that $\mu_{x^*} = \mu^*$. The sub-optimality gap of arm $x$ is $\Delta_x = \mu^* - \mu_x$ and 
$\Delta_{\min} = \min \set{\Delta_x : \Delta_x > 0, x \in \mathcal A}$
and $\Delta_{\max} = \max \set{\Delta_x : x \in \mathcal A}$.
The number of times arm $x$ has been chosen after round $t$ is denoted by $T_x(t) = \sum_{s=1}^t \ind{A_t = x}$ and
$T_*(t) = \sum_{s=1}^t \ind{\mu_{A_t} = \mu^*}$.
A policy $\pi$ is \textit{consistent} if for all $\theta$ and $p >0 $ it holds that $R^\pi_\theta(n) = o(n^p)$. Note that this is equivalent to $R^\pi_\theta(n) = O(n^p)$ and also to $\limsup_{n\to\infty}\log(R^\pi_\theta(n))/\log(n)\le 0$.
\todot{how is Landau notation more precise?} \todoc{Indeed, this is misleading. The other is also called Landau notation. Maybe more "faithful"?}  When more appropriate, we will use the more precise Landau notation $a_n \in O(b_n)$ (also with $\Omega$, $o$ and $\omega$).
Vectors in $\R^k$ will often be indexed by the action set, which we assume has an arbitrary fixed order. 
For example, we might write $\alpha \in \R^k$ and refer to $\alpha_x \in \R$ for some $x \in \cA$.

\input{lower-bound}

\input{upper-bound}

\section{SUB-OPTIMALITY OF OPTIMISM AND THOMPSON SAMPLING}\label{sec:optimism}

We now argue that algorithms based on optimism or Thompson sampling cannot be close to asymptotically optimal.
In each round $t$ an optimistic algorithm constructs a confidence set $\cC_t \subseteq \R^d$ and chooses $A_t$ according to
$A_t = \argmax_{x \in \mathcal A} \max_{\tilde \theta \in \cC_t} \shortinner{x, \tilde \theta}$.
In order to proceed we need to make some assumptions on $\cC_t$, otherwise one can define a ``confidence set'' to ensure any behaviour at all. 
First of all, we will assume that $\P{\exists t \leq n : \theta \notin \cC_t} = O(1/n)$. That is, that the probability that the true parameter is ever
outside the confidence set is not too large. Second, we assume that $\cC_t \subseteq \cE_t$ where $\cE_t$ is the ellipsoid about the least squares estimator given by
\eq{
\cE_t = \set{\tilde \theta : \shortnorm{\hat \theta(t) - \tilde \theta}_{G_t}^2 \leq \alpha \log(n)}\,,
}
where $\alpha$ is some constant and $\hat \theta(t)$ is the empirical estimate of $\theta$ based on the observations so far.
Existing algorithms based on confidence all use such confidence sets. Standard wisdom when designing optimistic algorithms is to use the smallest confidence set possible, so
an alternative algorithm that used a different form of confidence set would normally be advised to use the intersection $\cC_t \cap \cE_t$, which remains valid with high
probability by a union bound.
If the optimistic algorithm is not consistent, then its regret is not logarithmic on some problem and so diverges relative to the optimal strategy. Suppose now that the algorithm
is consistent. Then we design a bandit on which its asymptotic regret is worse than optimal by an arbitrarily large constant factor.

Let $d = 2$ and $e_1 = (1,0)$ and $e_2 = (0,1)$ be the standard basis vectors. 
The counter-example (illustrated in Figure \ref{fig:example}) is very simple with $\mathcal A = \set{e_1, e_2, x}$ where $x = (1 - \epsilon, 8\alpha \epsilon)$. 
The true parameter is given by $\theta = e_1$, which means that $x^* = e_1$ and $\Delta_{e_2} = 1$ and $\Delta_x = \epsilon$.
Suppose a consistent optimistic algorithm has chosen $T_{e_2}(t-1) \geq 4\alpha \log(n)$ and that $\theta \in C_t$. Then,
\eq{
\max_{\tilde \theta \in C_t} \shortinner{e_2, \tilde \theta} 
&\leq \shortinner{e_2, \hat \theta(t-1)} +  \sqrt{\norm{e_2}_{G_t^{-1}}^2 \alpha \log(n)} \\
&< 2 \sqrt{\norm{e_2}_{G_t^{-1}}^2 \alpha \log(n)} \leq 1\,.
}
But because $\theta \in C_t$, the optimistic value of the optimal action is at least $\shortinner{e_1, \theta} = 1$,
which means that $A_t \neq e_2$.
We conclude that if $\theta \in C_t$ for all rounds, then the optimistic algorithm satisfies $T_{e_2}(t-1) \leq 1 + 4\alpha \log(n)$.
By the assumption that $\theta \in C_t$ with probability at least $1 - 1/n$ we bound 
$\E[T_{e_2}(n)] \leq 2 + 4\alpha \log(n)$.
By consistency of the optimistic algorithm and our lower bound (Theorem \ref{thm:lower}) we have
\eq{
\limsup_{n\to\infty} \log(n) \norm{x - e_1}^2_{\bar G_n^{-1}} \leq \frac{\epsilon^2}{2}\,,
}
Therefore by choosing $\epsilon$ sufficiently small we conclude that
$\limsup_{n\to\infty} \E[T_x(n)] / \log(n) = \Omega(1/\epsilon^2)$
and so the asymptotic regret of the optimistic algorithm is at least
\eq{
\limsup_{n\to\infty} \frac{R_\theta^{\text{\scalebox{0.8}{\scshape optimistic}}}(n)}{\log(n)} = \Omega\left(\frac{1}{\epsilon}\right)\,.
}
However, for small $\epsilon$ the optimal regret for this problem is $c(\mathcal A, \theta) = 128 \alpha^2$ and so by choosing $\epsilon \ll \alpha$ we
can see that the optimistic approach is sub-optimal by an arbitrarily large constant factor.
The intuition is that the optimistic algorithms very quickly learn that $e_2$ is a sub-optimal arm and stop playing it. But as it turns out,
the information gained by choosing $e_2$ is sufficiently valuable that an optimal algorithm should use it for exploration.

\begin{wrapfigure}[8]{r}{3.2cm}
\vspace{-1.2cm}
\hspace{-0.7cm}
\begin{tikzpicture}[scale=0.4]
\draw[->] (0,0) -- (6,0);
\draw[->] (0,0) -- (0,6);
\draw[->] (0,0) -- (5.4, 1.2);
\draw[densely dotted] (6,0) -- (5.4, 1.2);
\draw[densely dotted] (5.4,1.2) -- (0,6);
\draw[densely dotted] (5.4,1.2) -- (5.4,0);
\draw[densely dotted,->] (5.7,0.6) -- (7.7,1.6);
\node at (5.85, -0.4) {\scriptsize $\epsilon$};
\node[xshift=0.5cm,yshift=0.4cm] at (5.4,1.2) {\scriptsize $(1-\epsilon, 2\epsilon)$};
\node[xshift=0.5cm] at (6, 0) {\scriptsize $(1, 0)$};
\node[yshift=0.3cm] at (0.3, 6) {\scriptsize $(0, 1)$};
\end{tikzpicture}
\caption{Counter-example}\label{fig:example}
\end{wrapfigure}
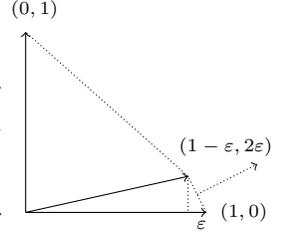
Thompson sampling has also been proposed for the linear bandit problem \citep{AG13}.
The standard approach uses a nearly flat Gaussian prior (and so posterior), which means that essentially the algorithm operates by sampling
$\theta_t$ from $\mathcal N(\hat \mu(t), \alpha G_t^{-1})$ and choosing the arm $A_t = \argmax_{x \in \mathcal A} \shortinner{x, \theta_t}$. 
Why does this approach fail? By the assumption of consistency we expect that the optimal arm will be played all but logarithmically often, which means that
the posterior will concentrate quickly about the value of the optimal action so that $\shortinner{x^*, \theta_t} \approx \mu^*$. Then using 
the same counter-example as for the optimistic algorithm we see that the likelihood that $\shortinner{e_2 - e_1, \theta_t} \geq 0$ is vanishingly small
once $T_{e_2}(t-1) = \Omega(\alpha \log(n))$ and so Thompson sampling will also fail to sample action $e_2$ sufficiently often.

\section{SUMMARY}\label{sec:summary}

We characterised the optimal asymptotic regret for linear bandits with Gaussian noise and finitely many actions in the sense of \cite{LR85}.
The results highlight a surprising fact that all reasonable algorithms based on optimism can be arbitrarily worse than optimal.
While this behaviour has been observed before in more complicated settings (notably, partial monitoring), our results are the first to illustrate this issue
in a setting only barely more complicated than finite-armed bandits.
Besides this we improve the self-normalised concentration guarantees by \cite{AST11} by a factor of $d$ asymptotically.

As usual, we open more questions than we answer.
While the proposed strategy is asymptotically optimal, it is also extraordinarily naive and the analysis is far from 
showing finite-time optimality.
For this reason we think the most pressing task is to develop efficient and practical algorithms that exploit the available information 
in a way that Thompson sampling and optimism do not. There are two natural research directions towards this goal. The first is to push the optimisation approach used 
here and also by \cite{Wu15}, but applied more ``smoothly'' without discarding data or long phases. The second is to generalise
information-theoretic ideas used (for instance) by \cite{RV14} or \cite{RCV16}.

\appendix

\bibliographystyle{plainnat}
\bibliography{all}

\ifsup
\input{concentration}

\input{tech}

\fi

\end{document}

%% file: lower-bound.tex
\section{LOWER BOUND}\label{sec:lower}
\newcommand{\bGn}{\bar G_n}
\newcommand{\xopt}{x^*}
\newcommand{\xoptT}{(x^*)^\top}
\newcommand{\Toptn}{T_*(n)}
\newcommand{\EToptn}{\EE{\Toptn}}
\newcommand{\ETxn}{\EE{T_x(n)}}
\newcommand{\EpToptn}{\EEp{\Toptn}}
\newcommand{\EpTxn}{\EEp{T_x(n)}}
\newcommand{\Deltamin}{\Delta_{\min}}
\newcommand{\bGnk}{\bar G_{n_k}}
\newcommand{\bGnj}{\bar G_{n_j}}
\newcommand{\tH}{\tilde{H}}

We note first that the finite-armed UCB algorithm of \cite{Agr95,KR95} can be used on this problem by disregarding the structure on the arms to achieve an
asymptotic regret of
\eq{
\limsup_{n\to\infty} \frac{R^{\smallsc{ucb}}_\theta(n)}{\log(n)} = \sum_{x \in \mathcal A:\Delta_x > 0} \frac{2}{\Delta_x}\,.
}
This quantity depends \emph{ linearly } on the number of suboptimal arms, which may be very large (much larger than the dimension) and is very undesirable.
Nevertheless we immediately observe that the asymptotic regret should be logarithmic. The following theorem and its corollary
characterises the optimal asymptotic regret.

\begin{theorem}\label{thm:lower}
Fix $\theta\in \R^d$ such that there is a unique optimal arm.
Let $\pi$ be a consistent policy and let 
\eq{
\bar G_n = \E\left[\sum_{t=1}^n A_t A_t^\top\right]\,,
}
which we assume is invertible for sufficiently large $n$.
Then for all suboptimal $x \in \mathcal A$ it holds that
\eq{
\limsup_{n\to\infty} \log(n) \norm{x-x^*}_{\bar G_n^{-1}}^2 \leq \frac{\Delta_x^2}{2}\,.
}
\end{theorem}
The astute reader may recognize $\norm{x-x^*}_{\bar G_n^{-1}}$ as the leading factor in the 
width of the confidence interval for estimating the gap $\Delta_x$ using a linear least squares estimator.
The result says that this width has to shrink at least logarithmically with a specific constant.
Before the proof of Theorem \ref{thm:lower} we present a trivial corollary and some consequences.
The assumption that $\bar G_n$ is eventually invertible can be relaxed. In fact, if $\bar G_n$ is not eventually invertible, then
the algorithm must suffer linear regret on some problem. This is quite natural because a singular $\bar G_n$ implies the algorithm
has not explored at all in some direction. The proof of this fact may be found in \ifsup Appendix \ref{app:singular}. \else the supplementary material. \fi

\begin{corollary}
\label{cor:regretlb}
Let $\pi$ be a consistent policy, $\theta\in \R^d$ such that there is a unique optimal arm in $\cA$. 
Then 
\begin{align}
\limsup_{n\to\infty} \log(n) \norm{x}_{\bar G_n^{-1}}^2 \leq \frac{\Delta_x^2}{2}\,
\label{eq:confwidth}
\end{align}
and also
$\displaystyle \limsup_{n\to\infty} \frac{R_\theta^\pi(n)}{\log(n)} \geq c(\mathcal A, \theta)$\,, \\[0.1cm]
where $c(\mathcal A, \theta)$ is defined as the solution to the following optimisation problem:
\begin{align}
\begin{split}
\inf_{\alpha\in [0,\infty)^{\cA}} \sum_{x\in \cA^{-}} \alpha(x) \Delta_x \text{ subject to } \\
\norm{x}^2_{H^{-1}(\alpha)} \leq \frac{\Delta_x^2}{2}\,, \quad \forall x\in \cA^-\,,
\end{split}
\label{eq:optproblem}
\end{align}
where $H(\alpha) = \sum_{x\in \cA} \alpha(x) x x^\top$.
\end{corollary}
As with the previous result, in~\eqref{eq:confwidth} the reader may recognize the leading term of the confidence width
for estimating the mean reward of $x$. Unsurprisingly, the width of this confidence interval has to shrink at least as fast
as the width of the confidence interval for estimating the gap $\Delta_x$.
The intuition underlying the optimisation problem~\eqref{eq:optproblem} is that no consistent strategy can escape allocating samples so that the gaps of all suboptimal actions are identified with high confidence, while a good strategy will also minimise the regret subject to the identifiability condition.
The proof of Corollary \ref{cor:regretlb} is given in \ifsup Appendix \ref{app:cor:regretlb}. \else the supplementary material. \fi

\begin{example}[Finite armed bandits]
Suppose $k = d$ and $\mathcal A =  \set{e_1,\ldots,e_k}$ be the standard basis vectors.
Then 
\eq{
c(\mathcal A, \theta) = \sum_{x \in \mathcal A : \Delta_x > 0} \frac{2}{\Delta_x}\,,
}
which recovers the lower bound by \cite{LR85}.
\end{example}

\begin{example}
Let $\alpha >1$ and $d = 2$ and $\cA = \set{x_1, x_2, x_3}$ with $x_1 = (1,0)$ and $x_2 = (0,1)$ and $x_3 = (1-\epsilon, \alpha \epsilon)$ and $\theta = (1,0)$.
Then $c(\mathcal A, \theta) = 2\alpha^2$ for all sufficiently small $\epsilon$.
The example serves to illustrate the interesting fact that $c(\mathcal A - \set{x_2}, \theta) = 2\epsilon^{-1} \gg c(\cA, \theta)$, which 
means that the problem becomes significantly harder if $x_2$ is removed from the action-set. The reason is that $x_1$ and $x_3$ are pointing in nearly the same
direction, so learning the difference is very challenging. But determining which of $x_1$ and $x_3$ is optimal is easy by playing $x_2$.
So we see that in linear bandits there is a complicated trade-off between information and regret that makes the structure of the optimal strategy more interesting
than in the finite setting.
\end{example}

The closest prior work to our lower bound is by \citet{KHN15} and \citet{AgTeAn89:pmon}. The latter consider stochastic partial monitoring when the 
reward is part of the observation. In this setting in each round, the learner selects one of finitely many actions and receives an observation from 
a distribution that depends on the action chosen and an unknown parameter, but is otherwise known. While this model could cover our setting, the 
results in the paper are developed only for the case when the unknown parameter belongs to a finite set, an assumption that all the results of the 
paper heavily depend on. \citet{KHN15} on the other hand restricts partial monitoring to the case when the observations belong to a finite set, 
while the parameter belongs to the unit simplex. While this problem also has a linear structure, their results do not generalize beyond the discrete 
observation setting.

\section{PROOF OF THEOREM \ref{thm:lower}}\label{sec:lower-proof}
\newcommand{\PP}{\bbP}

We make use of two standard results from information theory.
The first is a high probability version of Pinsker's inequality. 

\begin{lemma}\label{lem:kl}
Let $\mathbb P$ and $\mathbb P'$ be measures on the same measurable space $(\Omega,\cF)$. Then for any event $A\in \cF$,
\eqn{
\P{A} + \Pp{A^c} \geq \frac{1}{2} \exp\left(-\KL(\mathbb P, \mathbb P')\right)\,,
\label{eq:kl}
}
where $A^c$ is the complementer event of $A$ ($A^c = \Omega\setminus A$) and $\KL(\mathbb P, \mathbb P')$ is the relative entropy between $\PP$ and $\PP'$, which is defined as $+\infty$, if $\PP$ is not absolutely continuous with respect to $\PP'$, and is $\int_\Omega d\PP(\omega) \log \frac{d\PP}{d\PP'}(\omega)$ otherwise.
\end{lemma}
This result follows easily from Lemma 2.6 of \citet{Tsy08}.

The second lemma is sometimes called the information processing lemma and
shows that the relative entropy between measures on sequences of outcomes for the same algorithm interacting with different bandits can be
decomposed in terms of the expected number of times each arm is chosen and the relative entropies of the distributions of the arms. 
There are many versions of this result (e.g., \cite{ACFS95} and \cite{GL16}). 
To state the result, 
assume without the loss of generality that the measure space underlying the action-reward 
sequence $(A_1,Y_1,\dots,A_n,Y_n)$ is $\Omega_n  \doteq (\cA\times \R)^n$ and $A_t$ and $Y_t$ are the respective coordinate 
projections: $A_t(a_1,y_1,\dots,a_n,y_n) = a_t$ and $Y_t(a_1,y_1,\dots,a_n,y_n) = y_t$, $1\le t \le n$.

\begin{lemma}\label{lem:inf-processing}
Let $\mathbb P$ and $\mathbb P'$ 
be the probability measures on the sequence  $(A_1, Y_1,\ldots,A_n,Y_n)\in \Omega_n$ for a fixed
bandit policy $\pi$ interacting with a linear bandit with standard Gaussian noise and parameters $\theta$ and $\theta'$ respectively. Under these conditions the KL divergence of $\mathbb P$ and $\mathbb P'$ can be computed  exactly and is given by
\eqn{
\KL(\mathbb P, \mathbb P') = \frac12 \sum_{x \in \mathcal A} \E[T_x(n)]\, \shortinner{x, \theta - \theta'}^2\,,
}
where $\E$ is the expectation operator induced by $\PP$. 
\end{lemma}

\begin{proof}[Proof of Theorem \ref{thm:lower}]
Recall that $\xopt$ is the optimal arm, which we assumed to be unique. Let $x \in \cA$ be a suboptimal arm (so $\Delta_x > 0$) and
$A\subset \Omega_n$ be an event to be chosen later.
Rearranging~\eqref{eq:kl} gives $\KL(\PP,\PP') \ge \log(\frac{1}{2 \P{A}+2\Pp{A^c}})$
and recalling that
$\bar G_n = \E\left[\sum_{t=1}^n A_t A_t^\top\right]$, together with Lemma \ref{lem:inf-processing}
we get
 that 
\begin{align}\label{eq:chained}
\frac12 \norm{ \theta - \theta' }_{\bGn}^2 = \KL(\PP,\PP') \ge \log\left(\frac{1}{2 \P{A}+2\Pp{A^c}}\right)\,.
\end{align}
Now we choose $\theta'$ ``close'' to $\theta$, but in a such a way that $\inner{x - x^*, \theta'} > 0$, meaning in the bandit determined by $\theta'$ 
the optimal action is not $x^*$.
Selecting $A = \{ T_{x^*}(n)\le n/2 \}$ ensures that $\P{A} + \Pp{A^c}$ is small, because $\pi$ is consistent. Intuitively, this 
holds because if $\P{A}$ is large then $x^*$ is not used much in $\theta$, hence $R_n \doteq R_\theta^\pi(n)$ must be large. 
If $\Pp{A^c}$ is large, then $x^*$ is used often in $\theta'$, hence $R_n' \doteq R_{\theta'}^\pi(n)$ must be large.
But from the consistency of $\pi$ we know that both $R_n$ and $R_n'$ are sub-polynomial.
Let $\epsilon>0$ and $H\succeq 0$ ($H\in \R^{d\times d}$) to be chosen later and define $\theta'$ by 
\begin{align}
\theta' = \theta + \frac{H (x-\xopt)}{\norm{x-\xopt}_{H}^2 } (\Delta_{x}+\epsilon)\,,
\label{eq:thetapdef}
\end{align}
where we also restrict $H$ so that $\norm{x-\xopt}_H^2>0$.
Then, 
\begin{align}
\label{eq:suboptx}
\ip{x-\xopt,\theta'} = \ip{x-\xopt,\theta} + \Delta_x + \epsilon = \epsilon>0\,.
\end{align}
Hence the mean reward of $x$ is higher than that of $\xopt$ in $\theta'$.
\begin{align*}
R_n &= \sum_x \Delta_x \ETxn
          \ge \Deltamin \EE{ (n-\Toptn) }\\
        &\ge \Deltamin \,\EE{ \ind{ \Toptn\le n/2 } \frac{n}{2} } \\
        &= \frac{\Deltamin n}{2}\, \P{ \Toptn\le n/2}\,.
\end{align*}
On the other hand, introducing $\Delta'_y = \max_z \ip{z,\theta'} - \ip{y,\theta'}$ and $\Ep$ to denote the expectation
operator induced by $\PP'$ and using that by~\eqref{eq:suboptx}, $\xopt$ is suboptimal in $\theta'$, we also have
\begin{align*}
R_n' &=  \sum_x \Delta'_x \EpTxn \ge \Delta_{\xopt}' \EpToptn \\
& \ge \epsilon \EEp{ \ind{\Toptn>n/2} \Toptn }\\
& \ge \frac{\epsilon \, n}{2} \Pp{ \Toptn>n/2 }\,.
\end{align*}
Adding up the two inequalities and lower bounding $\epsilon+\Deltamin$ by $2\epsilon$, which holds when $\epsilon\le \Deltamin$ (which we assume from now on), we get 
\begin{align}\label{eq:psmall}
\frac{R_n + R_n' }{ \epsilon n} \ge  \P{  \Toptn\le \frac{n}{2}} + \Pp{ \Toptn>\frac{n}{2}} \,,
\end{align}
which completes the proof that $\P{  \Toptn\le n/2} + \Pp{ \Toptn> n/2}$ is indeed small.
Now we calculate the term on the left-hand side of~\eqref{eq:chained}. Using the definition of $\theta'$, we get
\begin{align*}
\frac{1}{2} \norm{\theta - \theta'}^2_{\bGn} 
&= \frac{(\Delta_x+\epsilon)^2}{2} \, \frac{\norm{x-\xopt}^2_{H \bGn H}}{ \norm{x-\xopt}_H^4}\\
& =\frac{(\Delta_x+\epsilon)^2}{2 \norm{s}_{\bGn^{-1}}^2} \, \rho_n(H)
\end{align*}
where in the last line we introduced
\begin{align*}
\rho_n(H) \doteq  \frac{\norm{s}_{\bGn^{-1}}^2 \,\norm{s}_{ H \bGn H }^2 }{ \norm{s}_H^4},\,\,
s = x-\xopt\,.
\end{align*}
Combining this with~\eqref{eq:psmall}, \eqref{eq:chained} and some algebra gives
\begin{align}
\frac{(\Delta_x+\epsilon)^2 \rho_n(H)}{2 \log(n) \norm{s}_{\bGn^{-1}}^2}
 \ge 1 - \frac{\log(\frac\epsilon{2})+\log(R_n+R_n')}{\log(n)}\,.
 \label{eq:finitesamplebound}
\end{align}
Since $\pi$ is consistent, $\limsup_{n\to\infty} \frac{\log(R_n+R_n')}{\log(n)}\le 0$.
Hence, for all $H\succeq 0$ such that $\norm{s}_H>0$,
\begin{align}
1 \le
\liminf_{n\to\infty} 
\frac{(\Delta_x+\epsilon)^2 \rho_n(H)}{2 \log(n) \norm{s}_{\bGn^{-1}}^2}\,.
\label{eq:liminflb}
\end{align}
Now take a subsequence $\{\bar G_{n_k}\}_{k=1}^\infty$ such that 
\eq{
c\doteq \limsup_{n\to\infty} \log(n) \norm{s}_{\bGn^{-1}}^2 = \lim_{k\to\infty} \log(n_k) \norm{s}_{\bGnk^{-1}}^2\,.
}
Hence,
\begin{align*}
\MoveEqLeft
\liminf_{n\to\infty} \frac{\rho_n(H)}{\log(n) \norm{s}_{\bGn^{-1}}^2}
\le
\liminf_{k\to\infty} \frac{\rho_{n_k}(H)}{\log(n_k) \norm{s}_{\bGnk^{-1}}^2}\\
&=
\liminf_{k\to\infty} \frac{\rho_{n_k}(H)}{\lim_{j\to\infty} \log(n_j) \norm{s}_{\bGnj^{-1}}^2}\\
&=
\frac{\liminf_{k\to\infty} \rho_{n_k}(H)}{c}\,. \numberthis \label{eq:liminflb2}
\end{align*}
Let $\tH_n = \bGn^{-1}/\norm{\bGn^{-1}}$. 
A simple calculation gives
$\rho_n(H)  = \norm{s}_{\tH_n}^2 \norm{s}_{H \tH_n^{-1} H}^2 \norm{s}_{H}^{-4}$ and hence if $H$ is any cluster point
of $\{\tH_{n_k}\}_k$, say, the subsequence $\{\tH_{n_k'}\}_k$ of the subsequence
$\{ \tH_{n_k}\}_k$ converges to $H$,
and $\norm{s}_H>0$ then
\begin{align*}
\MoveEqLeft
\liminf_{k\to\infty} 
\norm{s}_{\tH_{n_k}}^2 \norm{s}_{H \tH_{n_k}^{-1} H}^2 \norm{s}_{H}^{-4}\\
& \le
 \lim_{k\to\infty}  
 \norm{s}_{\tH_{n_k'}}^2 \norm{s}_{H \tH_{n_k'}^{-1} H}^2 \norm{s}_{H}^{-4}\\
& = \norm{s}_{H}^2 \norm{s}_{H H^{-1} H}^2 \norm{s}_H^{-4} = 1\,,
\end{align*}
showing that 
\begin{align*}
1\le 
\liminf_{n\to\infty} 
\frac{(\Delta_x+\epsilon)^2 \rho_n(H)}{2 \log(n) \norm{s}_{\bGn^{-1}}^2}
\le
\frac{(\Delta_x+\epsilon)^2 }{2 c}\,.
\end{align*}
Since $\epsilon>0$ was arbitrary small, the result will follow once we establish that $\norm{s}_H>0$.
To show this, assume on the contrary that $\norm{s}_H=0$. This implies that $Hs = 0$ and through $\ker(H) = \ker(H^{-1})$ it also implies that $H^{-1} s =0$.
Let $H_\gamma = H + \gamma I$, where $I$ is the $d\times d$ identity matrix.
Then, $H_\gamma s = \gamma s$, so $\norm{s}_{H_\gamma}^2 = \gamma \norm{s}>0$ and thus
\begin{align*}
\liminf_{k\to\infty} \rho_{n_k}(H_\gamma)
& \le \lim_{k\to\infty}  
 \norm{s}_{\tH_{n_k'}}^2 \norm{s}_{H_\gamma \tH_{n_k'}^{-1} H_\gamma}^2 \norm{s}_{H_\gamma}^{-4}\\
& = \lim_{k\to\infty}  
 \norm{s}_{\tH_{n_k'}}^2 \norm{s}_{ \tH_{n_k'}^{-1}}^2 \norm{s}^{-4}\\
& =   
 \norm{s}_{H}^2 \norm{s}_{ H^{-1}}^2 \norm{s}^{-4} = 0\,.
\end{align*}
Chaining \eqref{eq:liminflb}, \eqref{eq:liminflb2} and the last display gives $1\le 0$, a contradiction. 
Thus, $\norm{s}_H>0$ must hold, finishing the proof.
\end{proof}

\begin{remark}
The uniqueness assumption of the theorem can be lifted at the price of more work and by slightly changing 
the theorem statement. In particular, the theorem statement must be restricted to those suboptimal
actions $x\in \cA^-$ that can be made optimal by changing $\theta$ to $\theta'$, while none of the optimal actions $\cA^*(\theta) = \{ x\in \cA \,:\, \ip{x,\theta} = \max_{y\in \cA} \ip{y,\theta} \}$ are optimal. That is, the statement only concerns $x\in \cA$ such that $x\not\in \cA^*(\theta)$ but there exists $\theta'\in \R^d$ such that $\cA^*(\theta') \cap \cA^*(\theta)=\emptyset$ and $x\in \cA^*(\theta')$. The choice of $\theta'$ would still be as before, except that $x^*$ is selected as the optimal action under $\theta$ that maximizes $c(H,\theta)=\inf_{x'\in \cA^*(\theta)} \ip{x-x',x-x^*}_H$. Then, in the proof, $T_*(n)$ has to be redefined to be $\sum_{x\in \cA^*(\theta)} T_x(n)$ (the total number of times an optimal action is chosen), and at the end one also needs to show that the chosen $H$ satisfies $c(H,\theta)>0$. \todoc{This last step I have not verified.. If time permits.. Otherwise we can shorten this remark to basically say that we think an extension is possible.}
\end{remark}

%% file: upper-bound.tex
\section{CONCENTRATION}\label{sec:conc}

Before introducing the new algorithm we analyse the concentration properties of the least squares estimator. Our results refine
the existing guarantees by \cite{AST11}, and are necessary in order to obtain asymptotic optimality.
Let $G_t$ be the Gram matrix after round $t$ defined by $G_t = \sum_{s \leq t} A_s A_s^\top$
and $\hat \theta(t) = G_t^{-1} \sum_{s=1}^t A_s Y_s$ be the empirical (least squares) estimate, where $A_s$ is selected based on $A_1,Y_1,\dots,A_{s-1},Y_{s-1}$ and $Y_s = \ip{A_s,\theta} + \eta_s$, $\eta_s\sim N(0,1)$. \todoc{This is not super precise.}
We will only use $\hat \theta(t)$ for rounds $t$ when $G_t$ is invertible.
The empirical estimate of the sub-optimal gaps is 
$\hat \Delta_x(t) = \max_{y \in \mathcal A} \hat \mu_y(t) - \hat \mu_x(t)$,
where $\hat \mu_x(t) = \shortinner{x, \hat \theta(t)}$. We will also use the notation $\hat \mu(t)$ and $\hat \Delta(t) \in \R^k$ for vectors
of empirical means and sub-optimality gaps (indexed by the arms).

\begin{theorem}\label{thm:conc}
For any $\delta\in [1/n,1)$, $n$ sufficiently large and $t_0\in \N$ such that $G_{t_0}$ is almost surely non-singular, 
\eq{
\P{\exists t \geq t_0, x : \left|\hat \mu_x(t) - \mu_x\right| \geq \sqrt{\norm{x}_{G_t^{-1}}^2 f_{n,\delta}}} \leq \delta\,,
}
where for some $c>0$ universal constant
\eq{
f_{n,\delta} = 2\left(1 + \frac{1}{\log(n)}\right) \log(1/\delta) + c d \log(d \log(n))\,.
}
\end{theorem}

The result improves on the elegant concentration guarantee of \cite{AST11} because asymptotically we have $f_{n,1/n} \sim 2 \log(n)$, while there it was $2d \log(n)$.
Note that the restriction on $\delta$ may be relaxed with a small additional argument.
The proof of Theorem \ref{thm:conc} relies on a peeling argument and is given in \ifsup Appendix \ref{app:conc}. \else the supplementary material. \fi
For the remainder we abbreviate
$f_n = f_{n,1/n}$ and $g_n = f_{n,1/\log(n)}$, which are chosen so that
\eqn{
\label{eq:conc-cor} &\P{\exists t \geq t_0,x : \left|\hat \mu_x(t)  - \mu_x\right| \geq \sqrt{\norm{x}_{G_t^{-1}}^2 f_n}} \leq \frac{1}{n}\,, \\
&\P{\exists t \geq t_0,x : \left|\hat \mu_x(t) - \mu_x\right| \geq \sqrt{\norm{x}_{G_t^{-1}}^2 g_n}} \leq \frac{1}{\log(n)}\,. \nonumber
}

\section{OPTIMAL STRATEGY}\label{sec:upper}

A barycentric spanner of the action space is a set $B = \set{x_1,\ldots,x_d} \subseteq A$ such that for any $x \in \cA$ there exists an $\alpha \in [-1,1]^d$ with 
$x = \sum_{i=1}^d \alpha_i x_i$. The existence of a barycentric spanner is guaranteed because $\cA$ is finite and spans $\R^d$ \citep{AK04}.
We propose a simple strategy that operates in three phases called the \textit{warm-up} phase, the \textit{success} phase and the \textit{recovery} phase.
In the warm-up the algorithm deterministically chooses its actions from a barycentric spanner to obtain a rough estimate of the sub-optimality gaps. 
The algorithm then uses the estimated gaps as a substitute for the true gaps to determine the optimal pull counts for each action, and starts
implementing this strategy. Finally, if an anomaly is detected that indicates the inaccuracy of the estimated gaps then the algorithm switches to the recovery phase where
it simply plays UCB.

\begin{definition}\label{def:opt}
For any $\Delta \in [0,\infty)^k$ define $T_n(\Delta) \in [0,\infty]^k$ to be a solution to the optimisation problem
\begin{align*}
\MoveEqLeft
\min_{T \in [0,\infty]^k} \sum_{x \in \cA} T_x \Delta_x \text{ subject to } \\
&\norm{x}_{H_T^{\dagger}}^2 \leq \frac{\Delta^2_x}{f_n} \text{ for all } x \in \cA\,,\,\,
\text{where } H_T = \sum_{x \in \mathcal A} T_x x x^\top\,.
\end{align*}
\end{definition}

\begin{algorithm}[H]
\caption{Optimal Algorithm}\label{alg:optimal}
\begin{algorithmic}[1]
\State {\bf Input: } $\mathcal A$ and $n$
\State {\color{blue} // Warmup phase}
\State Find a barycentric spanner: $B = \set{x_1,\ldots,x_d}$
\State Choose each arm in $B$ exactly $\lceil\log^{1/2}(n)\rceil$ times 
\State {\color{blue} // Success phase}
\State $\displaystyle \epsilon_n \leftarrow \max_{x \in \mathcal A} \norm{x}_{G_{n}^{-1}} g_n^{1/2}$, $t \leftarrow n+1$
\State $\hat \Delta \leftarrow \hat \Delta(t-1)$ and $\hat T \leftarrow T_n(\hat \Delta)$ and $\hat \mu \leftarrow \hat \mu(t-1)$
\While{$t \le n$ and $\shortnorm{\hat \mu - \hat \mu(t-1)}_\infty \leq 2\epsilon_n$}
\State Play actions $x$ with $T_x(t) \leq \hat T_x$, $t\leftarrow t+1$
\EndWhile
\State {\color{blue} // Recovery phase}
\State Discard all data and play UCB until $t = n$.
\end{algorithmic}
\end{algorithm}

\begin{theorem}\label{thm:upper}
Assuming that $x^*$ is unique, the strategy given in Algorithm \ref{alg:optimal} satisfies
\eq{
\limsup_{n\to\infty} \frac{R^\pi_\theta(n)}{\log(n)} \leq c(\mathcal A, \theta) \text{ for all } \theta \in \R^d\,.
}
\end{theorem}

\section{PROOF OF THEOREM \ref{thm:upper}}\label{sec:upper-proof}

\newcommand{\Tsuccess}{T_{\text{succ.}}}
\newcommand{\Trecovery}{T_{\text{rec.}}}
\newcommand{\Twarmup}{T_{\text{warm.}}}

We analyse the regret in each of the three phases.
The warm-up phase has length $d \lceil\log^{1/2}(n) \rceil$, so its contribution to the asymptotic regret is negligible.
There are two challenges. The first is to show that the recovery phase happens with probability at most $1/\log(n)$.
Then, since the regret in the recovery phase is logarithmic by known results for UCB, this ensures that the expected regret incurred in the recovery phase is also negligible.
The second challenge is to show that the expected regret incurred during the success phase is asymptotically matching the lower bound in Theorem \ref{thm:lower}.

The set of rounds when the algorithm is in the warm-up/success/recovery phases are denoted by $\Twarmup$, $\Tsuccess$ and $\Trecovery$ respectively.
We introduce two failure events that occur when the errors in the empirical estimates of the arms are excessively large. 
Let $F_n$ be the event that there exists an arm $x$ and round $t\ge d$ such that
\eq{
\left|\hat \mu_x(t) - \mu_x\right| \geq \sqrt{\norm{x}_{G_t^{-1}}^2 g_n}\,.
}
Similarly, let $F_n'$ be the event that there exists an arm $x$ and round $t\ge d$ such that
\eq{
\left|\hat \mu_x(t) - \mu_x\right| \geq \sqrt{\norm{x}_{G_t^{-1}}^2 f_n}\,.
}
\cref{thm:conc} with $t_0=d$ and (\ref{eq:conc-cor}) imply that $\P{F_n} \leq 1/\log(n)$ and $\P{F_n'} \leq 1/n$.
The failure events determine the quality of the estimates throughout time. The following two lemmas show
that if $F_n$ does not occur then the regret is asymptotically optimal, while if $F'_n$ occurs then the regret is logarithmic with some constant
factor that depends only on the problem (determined by the action set $\mathcal A$ and the parameter $\theta$).
Since $F'_n$ occurs with probability at most $1/\log(n)$, the contribution of the latter component is negligible asymptotically.

\begin{lemma}\label{lem:F}
If $F_n$ does not occur then Algorithm \ref{alg:optimal} never enters the recovery phase. 
\todoc{$\Tsuccess$ should use mathrm, but to save space I don't switch to it now.} 
Furthermore,
\eq{
\limsup_{n\to\infty} \E\left[\frac{\ind{\text{not } F_n} \sum_{t \in \Tsuccess} \Delta_{A_t}}{\log(n)}\right] \leq c(\mathcal A, \theta)\,.
}
\end{lemma}

Before proving Lemma \ref{lem:F} we need a naive bound on the solution to the optimisation problem, the proof of which is given in \ifsup Appendix \ref{app:lem:opt}. \else the
supplementary material. \fi

\begin{lemma}\label{lem:opt}
Let $T = T_n(\Delta)$ for any $n$. Then
\eq{
\sum_{x : \Delta_x > 0} T_x \leq \frac{2d^3 f_n \Delta_{\max}}{\Delta_{\min}^3}\,.
}
\end{lemma}

\begin{proof}[Proof of Lemma \ref{lem:F}]
First, if $t = d \lceil{\log^{1/2}(n)}\rceil$ is the round at the end of the warm-up period
then by the definition of the
algorithm there is a barycentric spanner $B = \set{x_1,\ldots,x_d}$ and $T_{x_i}(t) = \lceil{\log^{1/2}(n)}\rceil$ for $1 \leq i \leq d$. 
Let $x \in \cA$ be arbitrary. Then, by the definition of the barycentric spanner, we can write $x = \sum_{i=1}^d \alpha_i x_i$ where $\alpha_i \in [-1,1]$ for all $i$.
Therefore,
\eq{
\norm{x}_{G_t^{-1}} \leq \sum_{i=1}^d \norm{x_i}_{G_t^{-1}} \leq \frac{d}{\log^{1/4}(n)}\,.
}
Recalling the definition of $\epsilon_n$ in the algorithm we have
\eq{
\epsilon_n = \max_{x \in \cA} \norm{x}_{G_n^{-1}} \sqrt{g_n} = O\left(\frac{d \log^{1/2}(\log(n))}{\log^{1/4}(n)}\right)\,.
}
Consider the case when $F_n$ does not hold. Then, for all arms $x$ and rounds $t$ after the warm-up period we have
\eq{
\left|\hat \mu_x(t) - \mu_x\right|
&\leq \norm{x}_{G_{t}^{-1}} \sqrt{g_n}
\leq \epsilon_n\,,
}
Therefore for all $s, t$ after the warm-up period we have
$|\hat \mu_x(t) - \hat \mu_x(s)| \leq 2\epsilon_n$,
which means the success phase never ends and so the first part of the lemma is proven. It remains to bound the regret.
Since we are only concerned with the asymptotics we may take $n$ to be large enough so that $2\epsilon_n \leq \Delta_{\min}/2$, which implies that $\hat \Delta_{x^*} = 0$.
For $T_n(\Delta)$, the solution to the optimisation problem in Definition \ref{def:opt} with the true gaps, it holds that 
\eqn{
\label{eq:opt-limit}
\limsup_{n\to\infty} \frac{\sum_{x \neq x^*} T_{n,x}(\Delta) \Delta_x}{\log(n)} = c(\cA, \theta)\,.
}
Letting $T^* = T_n(\Delta)$ and $1 + \delta_n = \max_{x : \hat \Delta_x > 0} \Delta_x^2 / \hat \Delta_x^2$, we have
\eq{
\norm{x}_{H_{(1 + \delta_n)T^*}^{-1}}^2 = \frac{\norm{x}_{H_{T^*}^{-1}}^2  }{1 + \delta_n}
\leq \frac{\Delta_x^2}{(1+\delta_n)f_n} \leq \frac{\hat \Delta_x^2}{f_n}\,.
}
Therefore, $\sum_{x\ne x^*} T_x \hat \Delta_x \le (1+\delta_n) \sum_{x\ne x^*} T^*_x \Delta_x$, where $T \doteq (T_x)_x \doteq T_x(n)$. Also,
\eqn{
&1 + \delta_n 
= \max_{x : \hat \Delta_x > 0} \frac{\Delta_x^2}{\hat \Delta_x^2} 
\leq \max_{x : \hat \Delta_x > 0} \frac{\Delta_x^2}{(\Delta_x - 2\epsilon_n)^2} \nonumber \\
&= \max_{x : \hat \Delta_x > 0} \left(1 + \frac{4(\Delta_x-\epsilon_n) \epsilon_n}{\left(\Delta_x - 2\epsilon_n\right)^2}\right) 
\leq 1 + \frac{16\epsilon_n}{\Delta_{\min}}\,,
\label{eq:delta}
}
where in the last inequality we used the fact that $0\le 2\epsilon_n \leq \Delta_{\min}/2$.
Then the regret in the success phase is
\eq{
&\sum_{t \in \Tsuccess} \Delta_{A_t}
\leq \sum_{x \neq x^*} T_x \Delta_x \\
&= \sum_{x \neq x^*} T_x \hat \Delta_x + \sum_{x \neq x^*} T_x (\Delta_x - \hat \Delta_x) \\
&\leq (1 + \delta_n) \sum_{x \neq x^*} T^*_x \hat \Delta_x + 2\epsilon_n \sum_{x\neq x^*} T_x  \\
&\leq (1 + \delta_n) \sum_{x \neq x^*} T^*_x \Delta_x + 2\epsilon_n \sum_{x \neq x^*} ((1 + \delta_n)T^*_x + T_x)\,. 
}
The result follows by taking the limit as $n$ tends to infinity and from Lemma \ref{lem:opt} and (\ref{eq:opt-limit}) and (\ref{eq:delta}), together with the reverse Fatou lemma.
\end{proof}

Our second lemma shows that provided $F_n'$ fails, the regret in the success phase is at most logarithmic:

\begin{lemma}\label{lem:Fp}
It holds that:
\eq{
\limsup_{n\to\infty} \frac{\E\left[\ind{F_n \text{ and not } F_n'} \sum_{t \in \Tsuccess} \Delta_{A_t}\right]}{\log(n)} = 0\,.
}
\end{lemma}

The proof follows by showing the existence of a constant $m$ that depends on $\cA$ and $\theta$, but not $n$ such that the regret suffered in
the success phase whenever $F_n'$ does not hold is almost surely at most $m \log(n)$. The result follows from this because $\P{F_n} \leq 1/\log(n)$. 
See \ifsup Appendix \ref{app:lem:Fp} for details. \else the supplementary material for details. \fi

\begin{proof}[Proof of Theorem \ref{thm:upper}]
We decompose the regret into the regret suffered in each of the phases:
\eqn{
R_\theta^\pi(n)  &=
\E\left[\sum_{\smash{t \in \Twarmup}} \Delta_{A_t}
 +\sum_{\smash{t \in \Tsuccess}}\Delta_{A_t}
 + \sum_{t \in \Trecovery} \Delta_{A_t}\right]\,.
\label{eq:decomp}
}
The warm-up phase has length $d \lceil\log^{1/2}(n) \rceil$, which contributes asymptotically negligibly to the regret:
\eqn{
\limsup_{n\to\infty} \frac{\E\left[\sum_{t \in \Twarmup} \Delta_{A_t}\right]}{\log(n)} = 0\,.
\label{eq:Twarmup}
}
By Lemma \ref{lem:F}, the recovery phase only occurs if $F_n$ occurs and $\P{F_n} \leq 1/\log(n)$. Therefore by well-known guarantees for UCB \citep{BC12}
there exists a universal constant $c > 0$ such that \todoc{Mildly fishy: need to argue that conditioning does not ruin things..}
\eq{
\E\left[\sum_{t \in \Trecovery} \Delta_{A_t}\right]
&= \E\left[\sum_{t \in \Trecovery} \Delta_{A_t}\Bigg| \Trecovery \neq \emptyset \right]
 \P{\Trecovery \neq \emptyset}  \\
&\leq \frac{ck \log(n)}{\Delta_{\min}} \P{\Trecovery \neq \emptyset} 
\leq \frac{ck}{\Delta_{\min}}\,.
}
Therefore
\eqn{
\label{eq:recovery}
\limsup_{n\to\infty} \frac{\E\left[\sum_{t \in \Trecovery} \Delta_t\right]}{\log(n)} = 0\,.
}
Finally we use the previous lemmas to analyse the regret in the success phase:
\eqn{
\E\left[\sum_{t \in \Tsuccess} \Delta_{A_t}\right] \nonumber
&= \E\left[\ind{\text{not } F_n} \sum_{t \in \Tsuccess} \Delta_{A_t}\right]  \\ \nonumber
&+ \E\left[\ind{F_n \text{ and not } F'_n} \sum_{t \in \Tsuccess} \Delta_{A_t}\right] \\ \label{eq:upper-last}
&+ \E\left[\ind{F'_n} \sum_{t \in \Tsuccess} \Delta_{A_t}\right] \,.
}
By (\ref{eq:conc-cor}), the last term satisfies
\eq{
&\limsup_{n\to\infty} \frac{\E\left[\ind{F'_n} \sum_{t \in \Tsuccess} \Delta_{A_t}\right]}{\log(n)} \\ 
&\qquad\leq \limsup_{n\to\infty} \frac{n \Delta_{\max} \P{F'_n}}{\log(n)}  = 0\,.
}
The first two terms in (\ref{eq:upper-last}) are bounded using Lemmas \ref{lem:F} and \ref{lem:Fp}, leading to
\eq{
\limsup_{n\to\infty} \frac{\E\left[\sum_{t \in \Tsuccess} \Delta_{A_t}\right]}{\log(n)} \leq c(\mathcal A, \theta)\,. 
}
Substituting the above display together with (\ref{eq:Twarmup}) and (\ref{eq:recovery}) into (\ref{eq:decomp}) completes the result.
\end{proof}

%% file: concentration.tex
\section{PROOF OF THEOREM \ref{thm:conc}}\label{app:conc}

Recall that $A_t$ is the action chosen in round $t$ and that $\eta_t = Y_t - \inner{A_t, \theta}$ is the noise term,
which we assumed to be a standard Gaussian. Let $S_t = \sum_{s=1}^t A_s \eta_s$.
By assumption, $\norm{A_t}\le 1$ for all $t\ge 1$.

\begin{lemma}
Let $n \in \N$ and $\epsilon > 0$ and $\sigma^2 > 0$.
Let $X_1,X_2,\ldots,X_n$ be a sequence of Gaussian random variables adapted to filtration $\cF_1,\cF_2,\ldots$ such that
$\E[X_t|\cF_{t-1}] = 0$. Define $\sigma^2_t = \Var[X_t|\cF_{t-1}]$ and assume that $\sigma^2_t \leq \sigma^2$ almost surely. Then 
\eq{
\P{\exists t \leq n : \sum_{s=1}^t X_s \geq \sqrt{2 \gamma_n V_t \log\left(\frac{N}{\delta}\right)}} \leq \delta\,,
}
where $V_t = \max\set{\epsilon, \sum_{s=1}^t \sigma^2_t}$ and 
\eq{
\gamma_n = 1 + \frac{1}{\log(n)} \quad \text{ and } \quad N = 1 + \ceil{\frac{\log(n\sigma^2/\epsilon)}{\log(\gamma_n)}}\,.
}
\end{lemma}

\begin{proof}
For $\psi \in \R$ define
\eq{
M_{t,\psi} = \exp\left(\sum_{s=1}^t \psi X_t - \frac{\psi^2\sigma_t^2}{2} \right)\,.
}
If $\tau \leq n$ is a stopping time with respect to $\cF$, then as in the proof \citep[Lemma 8]{AST11} we have 
$\E[M_{\tau,\psi}] \leq 1$.
Therefore, by Markov's inequality we have 
\eqn{
\label{eq:mix-mg}
\P{M_{\tau,\psi} \geq 1/\delta} \leq \delta\,.
}
For $k \in \set{1,2,\ldots,N}$ define 
\eq{
\psi_k = \sqrt{\frac{2}{\epsilon \gamma_n^{k-1}} \log\left(\frac{N}{\delta}\right)}
}
Then rearranging (\ref{eq:mix-mg}) leads to
\eq{
\P{\exists k \in [N] : \sum_{t=1}^\tau X_t \geq \frac{1}{\psi_k} \log\left(\frac{N}{\delta}\right) + \frac{\psi_k V_\tau}{2}} \leq \delta\,.
}
Therefore letting 
\eq{
k^* = \min\set{k \in [N] : \psi_k \geq \sqrt{2 \log(N/\delta) / V_\tau}}
}
leads to
\eq{
\delta 
&\geq \P{\sum_{t=1}^\tau X_t \geq \frac{1}{\psi_{k^*}} \log\left(\frac{N}{\delta}\right) + \frac{\psi_{k^*} V_\tau}{2}} \\
&\geq \P{\sum_{t=1}^\tau X_t \geq \sqrt{2 \gamma_n V_\tau \log\left(\frac{N}{\delta}\right)}}\,.
}
The result is completed by choosing stopping time $\tau$ by $\tau = \min(n,\tau_n)$, where
\eq{
\tau_n = \min\set{t \leq n : \sum_{s=1}^t X_s \geq \sqrt{2\gamma_n V_t \log\left(\frac{N}{\delta}\right)}}\,.
}
\end{proof}

\begin{lemma}\label{lem:conc-simple}
Let $\delta \in [1/n,1)$ and $\lambda \in \R^d$ with $\norm{\lambda} \leq 1$. 
Then
\eq{
\P{\exists t \leq n : \inner{\lambda, S_t} \geq \sqrt{\frac{1}{n^2} \vee \norm{\lambda}_{G_t}^2 h_{n,\delta}}} \leq \delta\,,
}
where
\eq{
h_{n,\delta} 
= 2\left(1 + \frac{1}{\log(n)}\right) \log\left(\frac{c \log(n)}{\delta}\right)\,
}
with some  universal constant $c \geq 1$.
\end{lemma}

\begin{proof}
We prepare to use the previous lemma. First note that
\eq{
\inner{\lambda, S_t} &= \sum_{s=1}^t \eta_s \inner{\lambda, A_t}\,.
}
Since $\eta_s$ is a standard Gaussian, the predictable variance of the term inside the sum is $\sigma^2_t = \inner{\lambda, A_t}^2 \leq \norm{\lambda}^2 \norm{A_t}^2 \leq 1$.
Therefore
\eq{
\sum_{s=1}^t \sigma^2_s 
&= \lambda^\top \sum_{s=1}^t A_s A_s^\top \lambda 
= \norm{\lambda}_{G_t}^2\,.
}
Therefore the result follows by the previous lemma with $X_t = \eta_t \inner{\lambda, A_t}$ and $\epsilon = 1/(n^2 \log(n)^3)$ and $\sigma^2 = 1$.
\end{proof}

The following lemma can be extracted from the proof of Theorem 1 in \cite{AST11}.

\begin{lemma}\label{lem:weak-conc}
Assume that $\{A_s\}$ is such that for some $t_0>0$, $G_{t_0}$ is non-singular almost surely.
Then, for some $c>0$ universal constant,
\eq{
\P{\exists t \geq t_0 : \norm{S_t}^2_{G_t^{-1}} \geq c d \log(n/\delta)} \leq \delta\,.
}
\end{lemma}

\begin{proof}[Proof of Theorem \ref{thm:conc}]
Let $\epsilon > 0$ be some small real number to be tuned subsequently and
choose $\cC \subset \R^d$ to be a finite covering set such that for all $x \in \cA$ and $t$ with $G_t$ non-singular there exists a $\lambda \in \cC$
such that $\lambda = (I + \cE)G_t^{-1} x$, where $\cE$ is some diagonal matrix (possibly depending on $x$ and $G_t^{-1}$) with entries bound in $[0, \epsilon]$.
Of course $G_t$ is a random variable, so we insist the existence of $\lambda$ is almost sure (that is, no matter how the actions are taken).
We defer calculating the necessary size $N = |\cC|$ until later. Let $\delta_1 = \delta / (N+1)$ and $F_\lambda$ be the event that
\eq{
F_\lambda = \set{\exists t : \inner{\lambda, S_t} \geq \sqrt{\frac{1}{n^2} \vee \norm{\lambda}_{G_t}^2 h_{n,\delta_1}}}\,.
}
Then a union bound and Lemma \ref{lem:conc-simple} leads to
\eqn{
\label{eq:F1}
\P{\cup_{\lambda \in \cC} F_\lambda} \leq N \delta_1\,.
}
By Lemma \ref{lem:weak-conc},
for $\mathcal G = \{\exists t \geq t_0 : \norm{S_t}^2_{G_t^{-1}} \geq c d \log(n/\delta_1)\}$,
 we have 
\eqn{
\label{eq:F2}
\P{\mathcal G} \leq \delta_1\,.
}
Another union bound shows that the $\P{\cup_{\lambda \in \cC} F_\lambda \cup \mathcal G} \le (N+1)\delta_1 = \delta$.
From now on we assume that neither $\mathcal F \doteq \cup_{\lambda \in \cC} F_\lambda$, nor $ \mathcal G$ occurs and let $x \in \mathcal A$ be arbitrary and for $t \geq t_0$
let $\lambda \in \cC$ be such that $\lambda =  (I + \cE) G_t^{-1} x$ where $\cE$ is diagonal with entries in $[0,\epsilon]$. Then 
\eqn{
&\hat \mu_x(t) - \mu_x 
=\inner{G_t^{-1} x, S_t} \\
&= \inner{G_t^{-1} x - \lambda, S_t} + \inner{\lambda, S_t} \nonumber \\
&\leq \norm{G_t^{-1} x - \lambda}_{G_t} \norm{S_t}_{G_t^{-1}} + \sqrt{\frac{1}{n^2} \vee \norm{\lambda}_{G_t}^2 h_{n,\delta_1}}\,. \label{eq:conc1}
}
We bound each term separately using matrix algebra and the assumption that the failure events $\mathcal F$ and $\mathcal G$ do not occur:
\eq{
\norm{G_t^{-1} x - \lambda}_{G_t}
&= \norm{\cE G_t^{-1} x}_{G_t} \\
&= \shortnorm{G_t^{1/2} \cE G_t^{-1/2} G_t^{-1/2} x} \\
&\leq \shortnorm{G_t^{1/2} \cE G_t^{-1/2}}_F \norm{x}_{G_t^{-1}}\,,
}
where $\norm{\cdot}_F$ is the Frobenius norm.
Then
\eq{
\shortnorm{G_t^{1/2} \cE G_t^{-1/2}}_F 
&= \sqrt{\trace(G_t \cE G_t^{-1} \cE)} \\ 
&\leq \sqrt{d} \norm{\cE}_\infty 
\leq \epsilon \sqrt{d}\,.
}
Therefore if $\epsilon = 1/(d^{3/2} \log(n))$, then the first term in (\ref{eq:conc1}) is bounded by
\eqn{
\norm{G_t^{-1} x - \lambda}_{G_t} \norm{S_t}_{G_t^{-1}} = O(1) \cdot \norm{x}_{G_t^{-1}}\,.
\label{eq:conc2}
}
For the second term we proceed similarly:
\eq{
\norm{\lambda}_{G_t}^2 
&= \norm{G_t^{-1} x + \cE G_t^{-1} x}_{G_t}^2  \\
&\leq \norm{x}_{G_t^{-1}}^2 \left(1 + \epsilon \sqrt{d}\right)^2 \\
&= (1 + o(1)) \norm{x}_{G_t^{-1}}^2\,.
}
Therefore, assuming $n$ is large enough so that $1/n^2 \leq \norm{x}/n \leq \norm{x}_{G_t^{-1}}^2$ (in the unique case that $\norm{x} = 0$ we simply
note that the following equality holds trivially), we have
\eq{
\sqrt{\frac{1}{n^2} \vee \norm{\lambda}_{G_t}^2 h_{n,\delta_1}}
= (1 + o(1)) \sqrt{\norm{x}_{G_t^{-1}}^2 h_{n,\delta_1}}\,.
}
Substituting the above expression along with (\ref{eq:conc2}) into (\ref{eq:conc1}) leads to
\eq{
\hat \mu_x(t) - \mu_x = (1 + o(1)) \sqrt{\norm{x}_{G_t^{-1}}^2 h_{n,\delta_1}}\,.
}
Finally we note that $\cC$ can be chosen in such a way that for suitably large universal constant $c > 0$ its size is
$\log N = O(d \log d \log(n))$.
This follows by treating each arm $x \in \cA$ separately and noting that $\norm{x}/n \leq \norm{G_t^{-1}x} \leq \norm{x}$.
Then letting $J = \ceil{\log(n)/\log(1+\epsilon)} = O(d^{3/2} \log^2(n))$, the covering set is given by $\cC = \bigcup_{x \in \cA} \cC_x$ where 
$\cC_x$ is a product covering space with a geometrical grid.
\eq{
\cC_x &= \bigtimes_{i=1}^d \set{\frac{\norm{x} (1 + \epsilon)^j}{n} : 0 \leq j \leq J}\,.
}
The theorem is completed by using the definition of $h_{n,\delta_1}$ in Lemma \ref{lem:conc-simple}. 
\end{proof}

%% file: tech.tex
\section{PROOF OF COROLLARY \ref{cor:regretlb}}\label{app:cor:regretlb}

Let $\cA^- = \cA \setminus \{\xopt\}$ be the set of suboptimal actions. To see~\eqref{eq:confwidth},
it suffices to show that for every consistent policy $\pi$ and vector $y\in \R^d$,
\begin{align}
\lim_{n\to\infty} \log(n)  y^\top \bGn^{-1} \xopt = 0\,.
\label{eq:keyc}
\end{align}

The proof hinges on the fact that $\EToptn\in \Omega(n)$ and for $x\in \cA^-$, $\ETxn \in \cap_{p>0} O(n^p)$.
Indeed, these follow from the assumption that $\pi$ is consistent and as such for any $p>0$,
$O(n^p) \ni R_\theta^\pi(n) = \sum_{x\in \cA^{-}} \Delta_x \ETxn$, so $\ETxn \in \cap_{p>0} O(n^p)$ indeed,
and thus also $\EToptn\in \Omega(n)$.

Let us return to proving \eqref{eq:keyc}. Clearly, it is enough to see this in the two cases: 
when $y=\xopt$ and when $y$ and $\xopt$ are perpendicular.
Consider first when $y = \xopt$. Then, from 
$\bGn \succeq \EToptn \xopt \xoptT$ it follows that $\bGn^{-1} \preceq (\EToptn)^{-1} \xopt \xoptT$ and
hence $\log(n) \xoptT \bGn^{-1} \xopt \le \frac{\log(n)}{\EToptn} \norm{\xopt}^2 \to 0$ as $n\to\infty$.

Now consider the case when $y$ and $\xopt$ are perpendicular. Let $v = \bGn^{-1} y$. Then, it must hold that $\bGn v = y$. 
Using the definition of $\bGn$, $y = \EToptn \xopt\xoptT v + \sum_{x\in \cA^{-}} \ETxn x x^\top v$. Since by assumption, $y$ and $\xopt$ are perpendicular, $0 = \xoptT y = \EToptn \norm{\xopt}^2 \xoptT v + \sum_{x\in \cA^{-}} \ETxn \xoptT x x^\top v$.
Hence, 
\begin{align*}
\log(n) \xoptT v = - \log(n) \sum_{x\in \cA^{-}} \frac{ \ETxn}{ \EToptn } \frac{\xoptT x x^\top v}{\norm{\xopt}^2} 
\end{align*}
converges to zero as $n\to\infty$.
This finishes the proof of  \eqref{eq:keyc} and thus of \eqref{eq:confwidth}.

For the second part we start with 
\eq{
\frac{R_\theta^\pi(n) }{\log(n)}= \sum_{x\in \cA^-} \frac{ \ETxn  }{\log(n)} \Delta_x\,.
}
Then $\alpha_n(x) = \ETxn /\log(n)$ is asymptotically feasible for $n$ large. Indeed, 
$\bGn = \log(n) H(\alpha_n)$, hence $\bGn^{-1} = H^{-1}(\alpha_n)/\log(n)$ and so 
\eq{
\frac{\Delta_x^2}{2} \ge \limsup_{n\to\infty} \log(n) \norm{x}_{\bGn^{-1}}^2 =  \limsup_{n\to\infty} \norm{x}^2_{H^{-1}(\alpha_n)} \,.
}
Thus for any $\epsilon>0$ and $n$ large enough, $\norm{x}^2_{H^{-1}(\alpha_n)}\le \Delta_x^2/2+\epsilon$ and
also
\eq{
\frac{R_\theta^\pi(n) }{\log(n)} = \sum_{x\in \cA^-} \frac{ \ETxn  }{\log(n)} \Delta_x \ge c_\epsilon(\cA,\theta)\,,
}
where $c_\epsilon(\cA,\theta)$ 
is the solution to the optimisation problem~\eqref{eq:optproblem} where $\Delta_x^2/2$ is replaced by $\Delta_x^2/2+\epsilon$. 
Hence, $\liminf_{n\to\infty} \frac{R_\theta^\pi(n) }{\log(n)} \ge c_\epsilon(\cA,\theta)$. Since $\epsilon>0$ was arbitrary 
and $\inf_{\epsilon>0} c_\epsilon(\cA,\theta) = c(\cA,\theta)$, we get the desired result.
\hfill $\square$

\section{PROOF THAT THE GRAM MATRIX IS EVENTUALLY NON-SINGULAR}\label{app:singular}

Let $\pi$ be a consistent strategy and $\cA$ and $\theta$ be the action-set and parameter for a linear bandit.
Define $\cA' = \set{x : \E[\sum_{t=1}^n \ind{A_t = x}] > 0}$ to be the set of arms that are played at least once with 
non-zero probability. 
We proceed by contradiction.
Suppose that $\bar G_n$ is singular for all $n$. Then there exists an $x \in \cA$ such that $x \notin \laspan{\cA'}$. 
Decompose $x = y + z$ where $y \in \laspan{\cA'}$ and $z \in \laspan{\cA'}^\bot$ is non-zero and in the orthogonal complement of the
subspace spanned by $\cA'$. Therefore $\shortinner{w, z} = 0$ for all $w \in \cA'$.
Define an alternative bandit with the same action-set and parameter $\theta' = \theta + 2\Delta_{\max} z$.
Then $\shortinner{w, \theta - \theta'} = 0$ for all $w \in \cA'$. Therefore the bandits determined by $\theta$ and $\theta'$ appear identical to 
the algorithm, and in particular, $\E'[\sum_{t=1}^n \ind{A_t \notin \cA'}] = 0$, and yet by construction we have
\eq{
R^\pi_{\theta'}(n) \geq \Delta_{\max} \E'\left[\sum_{t=1}^n \ind{A_t \in \cA'}\right] = n\Delta_{\max}\,. 
}
Therefore the regret is linear for $\theta'$, which implies that $\pi$ is not consistent. Therefore for sufficiently large $n$ we have $\bar G_n$ is non-singular.

\section{PROOF OF LEMMA \ref{lem:opt}}\label{app:lem:opt}

Let $B \subseteq \cA$ be a barycentric spanner and let $S \in [0,\infty]^k$ be an alternative to $T$ given by
\eq{
S_x = \begin{cases}
\infty\,, & \text{if } x = x^*; \\
\frac{2d^2f_n}{\Delta_{\min}^2}\,, & \text{if } x \in B; \\
0\,, & \text{otherwise}\,.
\end{cases}
}
Then $\norm{x^*}_{H_s^{\dagger}} = 0$ and for $x^* \neq y \in \cA$ we have
\eq{
\norm{y}_{H_{S^{\dagger}}}^2 
&\leq \left(\sum_{x \in B} \norm{x}_{H_{S^{\dagger}}}\right)^2 \\
&\leq \left(\frac{\Delta_{\min}}{\sqrt{2 f_n}}\right)^2 
\leq \frac{\Delta_y^2}{2f_n}\,.
}
Therefore
\eq{
\sum_{x : \Delta_x > 0} T_x 
&\leq \frac{1}{\Delta_{\min}} \sum_{x:\Delta_x > 0} T_x \Delta_x \\
&\leq \sum_{x : \Delta_x > 0} S_x \Delta_x  
\leq \frac{2d^3 \Delta_{\max}f_n}{\Delta_{\min}^3}\,.
\tag*{$\square$}
}

\section{PROOF OF LEMMA \ref{lem:Fp}}\label{app:lem:Fp}

The proof of Lemma \ref{lem:Fp} requires one more technical result. 

\begin{lemma}\label{lem:t-bound}
Let $\epsilon > 0$ and recall the definition of $T_n(\hat \Delta)$ given in Definition \ref{def:opt}. For $m \in \N$ define 
\eq{
S_{n,m}(\hat \Delta) = \min\set{m f_n, T_n(\hat \Delta)}\,.
}
Then there exists an $m$ such that for all $n \in \N$ and $\hat \Delta \in [0,\infty)^k$ and $x \in \mathcal A$
\eq{
\norm{x}_{H_{S_{n,m}(\hat \Delta)}^{-1}}^2 \leq \max\set{\frac{\epsilon^2}{f_n},\, \frac{\hat \Delta_x^2}{f_n}}\,. 
}
\end{lemma}

\begin{proof}[Proof of Lemma \ref{lem:Fp}]
Assume that $F_n'$ does not hold. We consider three cases. 
\begin{description}
\item[\normalfont \textit{Case 1.}] $\hat \Delta_{x^*} > 0$.
\item[\normalfont \textit{Case 2.}] $\hat \Delta_{x^*} = 0$ and $\hat \Delta_{\min} > \Delta_{\min}/4$.
\item[\normalfont \textit{Case 3.}] $\hat \Delta_{x^*} = 0$ and $\hat \Delta_{\min} \leq \Delta_{\min}/4$.
\end{description}
The idea is to show that in each case the regret is at most logarithmic, with a leading constant that depends on $\theta$ and $\cA$, but 
not on the observed samples. Treating each case separately.

\paragraph{Case 1}
Recall that $\hat \Delta \in \R^k$ (indexed by the actions) is the empirical estimate of the sub-optimality gaps after the warm-up phase.
Let $x$ be the sub-optimal arm for which $\hat \Delta_x = 0$. By the definition of the optimisation problem this arm will be played in every while loop.
Let $t$ be the first round when for all $x$ it holds that
\eq{
\norm{x}_{G_t^{-1}}^2 \leq \max\set{\frac{\hat \Delta_x^2}{f_n},\, \frac{\Delta_{\min}^2}{16f_n}}\,.
}
By Lemma \ref{lem:t-bound} there exists a constant $m_1$ depending only on $\mathcal A$ and $\theta$ such that
\eq{
t \leq m_1 f_n \,.
}
By the assumption that $F_n'$ does not hold (and its definition) we have
\eq{
\hat \mu_{x^*}(t) 
&\geq \mu_{x^*} - \max\set{\hat \Delta_{x^*},\, \Delta_{\min}/4} \\
&\geq \mu_x + \Delta_x - \hat \Delta_{x^*} - \frac{\Delta_{\min}}{4} \\
&\geq \hat \mu_x(t) + \Delta_x - \frac{\Delta_{\min}}{2} - \hat \Delta_{x^*} \\
&\geq \hat \mu_x(t) + \frac{\Delta_{\min}}{2} + \hat \mu_{x^*}(t_0) - \hat \mu_x(t_0)\,,
}
where $t_0 = d\lceil\log^{1/2}(n) \rceil$ is the round at the end of the warm-up phase.
Therefore if $n$ is sufficiently large that $\Delta_{\min}/2 \geq 4\epsilon_n$, then 
\eq{
\hat \mu_{x^*}(t) - \hat \mu_{x^*}(t_0) + \hat \mu_x(t_0) - \hat \mu_x(t) \geq \frac{\Delta_{\min}}{2} \geq 4\epsilon_n\,,
}
which by the fact that $\max\set{a,b} \geq (a + b)/2$ for all $a, b \in \R$ implies that the success phase of the algorithm ends.
Therefore if $n$ is sufficiently large, then in case 1 the regret in the success phase is at most  
\eqn{
\label{eq:case1}
\sum_{t \in \Tsuccess} \Delta_{A_t} \leq \Delta_{\max} m_1 f_n\,. 
}

\paragraph{Case 2}
Recall that $\hat T$ is the strategy used in the success phase based on samples collected in the warm-up phase.
Since $\hat \Delta_{x^*} = 0$ and $\hat \Delta_{\min} \geq \Delta_{\min}/4$, by Lemma \ref{lem:opt} it holds that
\eq{
\sum_{x \neq x^*} \hat T_x \leq \frac{2 \cdot 4^3 d^3 f_n \Delta_{\max}}{\Delta_{\min}^3}\,.
}
And again we have that for sufficiently large $n$ that the regret in the success phase is at most
\eqn{
\label{eq:case2}
\sum_{t \in \Tsuccess} \Delta_{A_t} \leq \frac{2 \cdot 4^3 d^3 f_n \Delta_{\max}^2}{\Delta_{\min}^3}\,.
}

\paragraph{Case 3}
For the final case
we assume that $\hat \Delta_{x^*} = 0$ and there exists an $x$ for which $\hat \Delta_x \leq \Delta_{\min}/4$.
Let $t$ be the first time-step when for all $x \in \mathcal A$ it holds that
\eq{
\norm{x}_{G_t^{-1}}^2 \leq \max\set{\frac{\Delta_{\min}^2}{64f_n},\, \frac{\hat \Delta_x^2}{f_n}}
}
Then by Lemma \ref{lem:t-bound} there exists a constant $m_2$ that is independent of $\hat \Delta$ and $n$ such that $t \leq m_2 f_n$.
Then since $F_n'$ does not hold we have
\eq{
\hat \mu_{x^*}(t) - &\hat \mu_{x^*}(t_0) + \hat \mu_{x}(t_0) - \hat \mu_x(t) \\
&\geq \hat \mu_{x^*}(t) - \hat \mu_x(t) - \hat \Delta_x 
\geq \frac{\Delta_{\min}}{4} - \hat \Delta_x \\
&\geq \frac{\Delta_{\min}}{2} 
\geq 2\epsilon_n\,.
}
Therefore provided that $n$ is sufficiently large, the success phase ends and by the same reasoning as in Case 1 the regret in the success phase is bounded by 
\eqn{
\label{eq:case3}
\sum_{t \in \Tsuccess} \Delta_{A_t} \leq \Delta_{\max} m_2 f_n\,.
}
The proof of the lemma is completed by combining (\ref{eq:case1}), (\ref{eq:case2}) and (\ref{eq:case3}), which imply the existence of a constant $m_3$ that 
is independent of $n$ and $\hat \Delta$ such that
\eq{
\ind{\text{not }F_n'} \sum_{t \in \Tsuccess} \Delta_{A_t} \leq m_3 f_n\,.
}
Therefore by (\ref{eq:conc-cor}) and the definition of $f_n \sim 2 \log(n)$ we have
\eq{
&\limsup_{n\to\infty} \frac{\E\left[\ind{F_n \text{ and not } F_n'} \sum_{t \in \Tsuccess} \Delta_{A_t}\right]}{\log(n)} \\
&\leq \limsup_{n\to\infty} \frac{\E[\ind{F_n} m_3 f_n]}{\log(n)} \\
&= \limsup_{n\to\infty} \frac{\P{F_n} m_3 f_n}{\log(n)} \\
&\leq \limsup_{n\to\infty} \frac{m_3 f_n}{\log^2(n)} \\
&= 0\,.
\qedhere
}
\end{proof}